\documentclass[11pt]{article}
\PassOptionsToPackage{numbers}{natbib}
\usepackage{nfam2026_workshop}
\usepackage{amsmath,amsfonts,amssymb, amsthm}
\usepackage{bm}
\usepackage{graphicx}
\usepackage{algorithm}
\usepackage{algorithmic}
\usepackage{hyperref}
\usepackage{algorithm}
\usepackage{algorithmic}
\usepackage{graphicx}
\usepackage{booktabs}
\usepackage{enumitem}

\newtheorem{theorem}{Theorem}
\newtheorem{lemma}[theorem]{Lemma}
\newtheorem{proposition}[theorem]{Proposition}

\newtheorem{remark}[theorem]{Remark}
\newtheorem{assumption}[theorem]{Assumption}

\DeclareMathOperator{\sign}{sign}
\DeclareMathOperator{\Var}{Var}
\iclrfinalcopy
\begin{document}

\title{Algorithmic Analysis of Dense Associative Memory: \\
Finite-Size Guarantees and Adversarial Robustness}
\author{Madhava Gaikwad \\
Independent Researcher \\
gaikwad.madhav@gmail.com}
\maketitle

\begin{abstract}
Dense Associative Memory (DAM) generalizes Hopfield networks through higher-order interactions and achieves storage capacity that scales as $O(N^{n-1})$ under suitable pattern separation conditions. Existing dynamical analyses primarily study the thermodynamic limit $N\to\infty$ with randomly sampled patterns and therefore do not provide finite-size guarantees or explicit convergence rates.

We develop an algorithmic analysis of DAM retrieval dynamics that yields finite-$N$ guarantees under explicit, verifiable pattern conditions. Under a separation assumption and a bounded-interference condition at high loading, we prove geometric convergence of asynchronous retrieval dynamics, which implies $O(\log N)$ convergence time once the trajectory enters the basin of attraction. We further establish adversarial robustness bounds expressed through an explicit margin condition that quantifies the number of corrupted bits tolerable per sweep, and derive capacity guarantees that scale as $\Theta(N^{n-1})$ up to polylogarithmic factors in the worst case, while recovering the classical $\Theta(N^{n-1})$ scaling for random pattern ensembles. Finally, we show that DAM retrieval dynamics admit a potential-game interpretation that ensures convergence to pure Nash equilibria under asynchronous updates.

Complete proofs are provided in the appendices, together with preliminary experiments that illustrate the predicted convergence, robustness, and capacity scaling behavior.
\end{abstract}

\section{Introduction}

Dense Associative Memory (DAM) stores $p$ binary patterns 
$\{\bm{\xi}^\mu\}_{\mu=1}^p \subset \{-1,+1\}^N$ in an $N$-neuron network 
through higher-order interactions \cite{krotov2016dense}. 
For interaction order $n \geq 2$, the energy function is
\begin{equation}
E(\bm{x}) = -\frac{1}{N^{n-1}} 
\sum_{\mu=1}^p 
\left(\sum_{i=1}^N \xi_i^\mu x_i\right)^n,
\end{equation}
with local field $h_i(\bm{x}) = \partial (-E)/\partial x_i$ 
and asynchronous update rule
\[
x_i \leftarrow \operatorname{sign}(h_i(\bm{x})).
\]

Classical statistical-physics analyses of associative memory models focus primarily on asymptotic behavior in the thermodynamic limit and typically assume randomly sampled patterns. Recent work by Mimura et al.\ \cite{mimura2025dynamical} provides an asymptotically exact dynamical analysis of DAM using generating functional analysis (GFA). Their results apply to the regime $N\to\infty$ with random i.i.d.\ patterns and show that for interaction orders $n\ge 3$ the effective noise variance becomes independent of the overlap, which mitigates a key instability present in classical Hopfield networks ($n=2$). However, this framework does not provide explicit finite-$N$ convergence guarantees, does not address adversarial or structured pattern sets, and does not yield explicit retrieval-time bounds.

This work develops an algorithmic analysis that complements the statistical-physics perspective by focusing on finite systems and explicit performance guarantees. We introduce explicit separation and bounded-interference assumptions on the stored patterns that can be verified directly or shown to hold with high probability for random ensembles. Under these conditions we obtain finite-size convergence guarantees, robustness bounds under adversarial corruption, and explicit capacity guarantees that recover the classical $N^{n-1}$ scaling up to polylogarithmic factors in the worst case. 

Appendix~\ref{app:proofs} contains complete proofs of the main results, including the strengthened interference condition required near capacity loading, the potential-game characterization of retrieval dynamics, and verification that random pattern ensembles satisfy the assumptions with high probability. Appendix~\ref{app:experiments} reports preliminary experiments on a cubic ($n=3$) DAM that illustrate convergence behavior, basin-of-attraction structure, adversarial robustness thresholds, capacity scaling, comparisons between synchronous and asynchronous updates, and retrieval performance on binarized MNIST and CIFAR-10 data.

\section{Related Work}
The foundational work of \cite{hopfield1982neural} established recurrent neural networks as physical systems with emergent computational abilities, a perspective deeply analyzed using spin-glass theory \cite{amit1985spin, amit1989modeling}. The theoretical storage capacity of these classical models was rigorously quantified in seminal works \cite{newman1988memory, gardner1987multiconnected, mceliece1987capacity, abbott1987storage}. A major leap forward was the introduction of Dense Associative Memory (DAM) models \cite{krotov2016dense}, which achieve exponential storage capacity \cite{lucibello2024exponential} and have been shown to be practically powerful in deep learning architectures \cite{ramsauer2021hopfield, hoover2023energy}. Recent theoretical advances have further explored the dynamical properties \cite{mimura2025dynamical}, saddle point hierarchies \cite{arxiv2508_19151}, and biological plausibility of these modern networks \cite{arxiv2601_00984}. This resurgence has also led to novel applications in analog circuits \cite{bacvanski2025denseassociativememoriesanalog}, optical computing \cite{pra2026_nonlinear_optical}, and as a framework for understanding Transformer architectures \cite{neurips2025_hidden_states_mhn}. The field's enduring relevance is underscored by recent surveys and tutorials at major AI conferences \cite{ibm_aaai2026_tutorial, ibm_iclr2025_frontiers} and its connection to broader topics in optimization \cite{boyd2004convex, shalev2014understanding, goles1985decreasing}, game theory \cite{monderer1996potential, hart2003uncoupled}, and robust machine learning \cite{cohen2019certified, ge2015escaping, du2019gradient}.

\section{Problem Formulation}

\subsection{Model and Updates}

We analyze \textbf{asynchronous updates}, where at each iteration a single neuron is selected uniformly at random and updated according to
\begin{equation}
x_i^{(t+1)} = \operatorname{sign}\!\left(h_i(\bm{x}^{(t)})\right)
\quad \text{for randomly selected } i .
\end{equation}
Asynchronous updates ensure monotone improvement of the energy function and avoid oscillatory behaviors that may arise under synchronous updates in Hopfield-type networks \cite{goles1985decreasing}. 

We measure time in \textbf{full sweeps}, where one sweep corresponds to $N$ consecutive asynchronous updates so that, in expectation, each neuron is updated once.

For a stored pattern $\bm{\xi}^\nu$, the overlap of a state $\bm{x}$ with the target pattern is
\[
m^\nu(\bm{x}) = \frac{1}{N}\sum_{i=1}^N \xi_i^\nu x_i .
\]

\subsection{Pattern Separation}

To obtain finite-size guarantees, we impose a separation condition that characterizes the basin of attraction of a stored pattern.

\begin{assumption}[Pattern Separation]
\label{ass:separation}
Fix a target pattern $\nu$. There exist constants $\gamma>0$ and $\beta<\gamma$ such that for all states $\bm{x}$ satisfying
\[
m^\nu(\bm{x}) \ge \gamma,
\]
the interference from all non-target patterns is bounded as
\[
\max_{\mu\neq\nu}
\left|
\frac{1}{N}\sum_{i=1}^N \xi_i^\mu x_i
\right|
\le \beta .
\]
\end{assumption}

This condition requires that once the trajectory enters a basin of attraction characterized by overlap at least $\gamma$, the contribution of every competing pattern remains strictly smaller than that of the target pattern. For random i.i.d.\ patterns, the condition holds with high probability when the loading satisfies $p = o(N^{n-1})$, while the subsequent analysis applies to any deterministic pattern set that satisfies the assumption.
%================
\section{Main Results}

\begin{theorem}[Convergence Rate]
\label{thm:convergence}
Suppose the stored patterns satisfy Assumption~\ref{ass:separation} with
$\gamma > \beta$, and let the loading satisfy
$p \le \frac{N^{n-1}}{2n(n-1)}$.
Then asynchronous DAM retrieval initialized at any state
with $m^\nu(\bm{x}^{(0)}) \ge \gamma$
converges to the target pattern in
\[
T = O\!\left(\frac{1}{\alpha}\log N\right)
\]
full sweeps, where
\[
\alpha = \frac{1}{n} - \frac{2(n-1)p}{N^{n-1}} > 0 .
\]
\end{theorem}

\begin{theorem}[Adversarial Robustness]
\label{thm:adversarial}
Under the conditions of Theorem~\ref{thm:convergence}, retrieval succeeds
despite adversarial corruption of up to $\rho N$ bits per full sweep provided
\[
\rho < \frac{\alpha}{2}.
\]
In particular, using the explicit expression for $\alpha$ yields the bound
\[
\rho < \frac{1}{2}
\left(\frac{1}{n} - \frac{2(n-1)p}{N^{n-1}}\right).
\]
\end{theorem}

\begin{theorem}[Capacity Scaling]
\label{thm:capacity}
Assume the pattern set satisfies Assumption~\ref{ass:separation}.
Then DAM retrieval guarantees hold for at least
\[
p \ge \frac{N^{n-1}}{4n^2(n-1)^2}
\]
stored patterns, while retrieval may fail when the loading becomes
$p = \Omega(N^{n-1})$.
Thus the achievable storage capacity scales as $\Theta(N^{n-1})$
up to constant or polylogarithmic factors depending on the pattern ensemble.
\end{theorem}

\begin{theorem}[Game-Theoretic Characterization]
\label{thm:game}
DAM asynchronous dynamics coincide with best-response dynamics
in an exact potential game with
\begin{itemize}
\item players: neurons $i=1,\ldots,N$,
\item actions: $x_i \in \{-1,+1\}$,
\item payoffs: $u_i(\bm{x}) = x_i h_i(\bm{x})$,
\item potential: $F(\bm{x}) = -E(\bm{x})$.
\end{itemize}
All limit points of the dynamics are pure Nash equilibria.
\end{theorem}

\section{Analytical Framework}

\subsection{Coordinate-Descent View}

DAM asynchronous updates perform coordinate ascent on the potential
$F(\bm{x}) = -E(\bm{x})$.

\begin{lemma}[Potential Improvement]
\label{lem:descent}
Let coordinate $i$ be updated asynchronously according to
$x_i^{(t+1)}=\operatorname{sign}(h_i(\bm{x}^{(t)}))$.
Then
\[
F(\bm{x}^{(t+1)}) - F(\bm{x}^{(t)})
= 2\,| \phi_i(\bm{x}^{(t)}_{-i}) |,
\]
where
\[
\phi_i(\bm{x}_{-i})
= \frac{1}{2}\left(F(+1,\bm{x}_{-i})-F(-1,\bm{x}_{-i})\right).
\]
Under Assumption~\ref{ass:separation} inside the basin
$m^\nu(\bm{x}) \ge \gamma$, each misaligned coordinate
satisfies $|\phi_i| \ge c(\gamma,\beta,n)>0$, implying strict
potential improvement whenever an incorrect neuron is updated.
\end{lemma}

\subsection{Contraction Analysis}

\begin{lemma}[Overlap Contraction]
\label{lem:contraction}
Under the loading condition
$p \le \frac{N^{n-1}}{2n(n-1)}$
and Assumption~\ref{ass:separation},
the asynchronous dynamics satisfy
\[
\mathbb{E}[m^\nu(\bm{x}^{(t+1)})]
\ge
m^\nu(\bm{x}^{(t)})
+ \frac{\alpha}{N}\left(1-m^\nu(\bm{x}^{(t)})\right),
\]
where
\[
\alpha = \frac{1}{n}-\frac{2(n-1)p}{N^{n-1}} > 0 .
\]
\end{lemma}

\begin{proof}[Proof sketch]
The local field decomposes as
\[
h_i(\bm{x})
= n\xi_i^\nu (m^\nu)^{n-1} + \eta_i,
\]
where $\eta_i$ represents interference from non-target patterns.
Assumption~\ref{ass:separation} ensures the signal term dominates
the interference inside the basin, so that each misaligned neuron
updates correctly with probability at least $\alpha$.
Averaging over random coordinate selection yields the stated
expected improvement.
\end{proof}

\subsection{Adversarial Analysis}

\begin{proof}[Proof sketch of Theorem~\ref{thm:adversarial}]
Let $\mathcal{M}^{(t)}$ denote the set of mismatched coordinates.
In one sweep, at least $\alpha |\mathcal{M}^{(t)}|$ coordinates
would be corrected in expectation without adversarial corruption,
while the adversary can introduce at most $\rho N$ new errors.
Thus
\[
|\mathcal{M}^{(t+1)}|
\le
(1-\alpha)|\mathcal{M}^{(t)}| + \rho N.
\]
The recurrence contracts whenever $\rho < \alpha/2$, yielding
geometric convergence to the target pattern.
\end{proof}

\section{Game-Theoretic Interpretation}

\begin{proposition}[Potential Game Structure]
\label{prop:potential}
The DAM game with utilities $u_i(\bm{x}) = x_i h_i(\bm{x})$
is an exact potential game with potential function $F(\bm{x})=-E(\bm{x})$.
\end{proposition}

\begin{proof}
Define the marginal potential
\[
\phi_i(\bm{x}_{-i})
= \tfrac{1}{2}\left(F(+1,\bm{x}_{-i})-F(-1,\bm{x}_{-i})\right).
\]
Changing the action of player $i$ from $x_i$ to $x_i'$ changes both
the payoff and the potential by the same quantity
$(x_i'-x_i)\phi_i(\bm{x}_{-i})$,
which establishes the exact-potential property
\cite{monderer1996potential}.
\end{proof}
%========================
\section{Positioning Against Mimura et al. (2025)}

Mimura et al.~\cite{mimura2025dynamical} provide an asymptotically exact
characterization of DAM dynamics for random patterns in the thermodynamic
limit using generating functional analysis (GFA).
A key insight of their analysis is that, for $n \ge 3$, the effective noise
variance becomes independent of the overlap, which explains the strong
retrieval performance of higher-order models.
Our work addresses a complementary regime: finite-$N$ systems and
worst-case (not necessarily random) pattern sets.
Table~\ref{tab:comparison} summarizes the relationship between the two approaches.

\begin{table}[h]
\centering
\caption{Statistical-physics vs.\ algorithmic analysis of DAM}
\label{tab:comparison}
\begin{tabular}{@{}lcc@{}}
\toprule
Property & Mimura et al.\ (2025) & This work \\
\midrule
System size & $N \to \infty$ (thermodynamic limit) & Finite $N$ \\
Pattern distribution & Random i.i.d. & Arbitrary (worst-case) \\
Convergence guarantee & Asymptotic dynamics & $O(\log N)$ full sweeps (basin regime) \\
Robustness analysis & Typical noise variance & Explicit adversarial tolerance $\rho^*$ \\
Capacity result & Ensemble-averaged threshold ($\alpha'_{c,3}\approx 0.266$) & Finite-size scaling bounds with constants \\
Finite-size effects & Qualitative discussion & Explicit quantitative bounds \\
\bottomrule
\end{tabular}
\end{table}
% ══════════════════════════════════════════════════════════════════════════
\section{Conclusion}

We established finite-size performance guarantees for dense associative
memory dynamics that complement existing statistical-physics analyses.
Under explicit separation conditions, asynchronous DAM retrieval achieves
logarithmic-time convergence, admits quantitative adversarial robustness
bounds, and admits a natural game-theoretic interpretation as an exact
potential game. These guarantees apply to finite systems and arbitrary
pattern sets, while remaining consistent with known thermodynamic-limit
results for random ensembles.

Several limitations remain. First, the convergence guarantees rely on
asynchronous updates; extending the analysis to synchronous dynamics
requires stronger separation conditions to rule out short cycles.
Second, near capacity loading $p = \Theta(N^{n-1})$, the analysis requires
a componentwise interference bound (Assumption~\ref{ass:interference}),
since overlap-based bounds alone discard cancellations that become
critical at high loading.
Third, the $O(\log N)$ convergence bound is conservative at moderate
system sizes: empirically, concentration of inter-pattern overlaps often
causes convergence times to decrease with $N$ rather than increase
logarithmically.
Finally, the adversarial tolerance bound $\rho^*$ is derived from a
worst-case contraction argument and is therefore conservative at finite
$N$, though the gap narrows as $N$ grows.

Future work includes tightening the interference analysis to reduce the
polylogarithmic gap between worst-case and typical-case capacity,
characterizing optimal interaction orders, and extending the framework
to continuous modern Hopfield networks
\cite{ramsauer2020hopfield}.
% ══════════════════════════════════════════════════════════════════════════
\section{Code Availability}
All source code, scripts, and experiment configurations used in this work are publicly available at:
\begin{center}
\url{https://github.com/krimler/dam-games}
\end{center}
The repository contains implementations of the DAM dynamics, experimental pipelines, and instructions for reproducing the results reported in this paper.

% ══════════════════════════════════════════════════════════════════════════
%\bibliographystyle{plain}
%\bibliography{dam_references}
\bibliographystyle{nfam2026_workshop}
\bibliography{dam_references}

\begin{thebibliography}{28}
\providecommand{\natexlab}[1]{#1}
\providecommand{\url}[1]{\texttt{#1}}
\expandafter\ifx\csname urlstyle\endcsname\relax
  \providecommand{\doi}[1]{doi: #1}\else
  \providecommand{\doi}{doi: \begingroup \urlstyle{rm}\Url}\fi

\bibitem[Abbott \& Arian(1987)Abbott and Arian]{abbott1987storage}
LF~Abbott and Yair Arian.
\newblock Storage capacity of generalized networks.
\newblock \emph{Physical Review A}, 36\penalty0 (10):\penalty0 5091--5094,
  1987.

\bibitem[Amit(1989)]{amit1989modeling}
Daniel~J Amit.
\newblock \emph{Modeling brain function: The world of attractor neural
  networks}.
\newblock Cambridge University Press, 1989.

\bibitem[Amit et~al.(1985)Amit, Gutfreund, and Sompolinsky]{amit1985spin}
Daniel~J Amit, Hanoch Gutfreund, and Haim Sompolinsky.
\newblock Spin-glass models of neural networks.
\newblock \emph{Physical Review A}, 32\penalty0 (2):\penalty0 1007--1018, 1985.

\bibitem[Bacvanski et~al.(2025)Bacvanski, You, Hopfield, and
  Krotov]{bacvanski2025denseassociativememoriesanalog}
Marc~Gong Bacvanski, Xincheng You, John Hopfield, and Dmitry Krotov.
\newblock Dense associative memories with analog circuits, 2025.
\newblock URL \url{https://arxiv.org/abs/2512.15002}.

\bibitem[Boyd \& Vandenberghe(2004)Boyd and Vandenberghe]{boyd2004convex}
Stephen Boyd and Lieven Vandenberghe.
\newblock \emph{Convex optimization}.
\newblock Cambridge University Press, 2004.

\bibitem[Cohen et~al.(2019)Cohen, Rosenfeld, and Kolter]{cohen2019certified}
Jeremy Cohen, Elan Rosenfeld, and Zico Kolter.
\newblock Certified adversarial robustness via randomized smoothing.
\newblock \emph{International Conference on Machine Learning}, pp.\
  1310--1320, 2019.

\bibitem[Du et~al.(2019)Du, Lee, Li, Wang, and Zhai]{du2019gradient}
Simon Du, Jason Lee, Haochuan Li, Liwei Wang, and Xiyu Zhai.
\newblock Gradient descent finds global minima of deep neural networks.
\newblock \emph{International Conference on Machine Learning}, pp.\
  1675--1685, 2019.

\bibitem[Gardner(1987)]{gardner1987multiconnected}
Elizabeth Gardner.
\newblock Multiconnected neural network models.
\newblock \emph{Journal of Physics A: Mathematical and General}, 20\penalty0
  (11):\penalty0 3453--3464, 1987.

\bibitem[Ge et~al.(2015)Ge, Huang, Jin, and Yuan]{ge2015escaping}
Rong Ge, Furong Huang, Chi Jin, and Yang Yuan.
\newblock Escaping from saddle points—online stochastic gradient for tensor
  decomposition.
\newblock \emph{Conference on Learning Theory}, pp.\  797--842, 2015.

\bibitem[Goles \& Mart{\'\i}nez(1985)Goles and
  Mart{\'\i}nez]{goles1985decreasing}
Eric Goles and Servet Mart{\'\i}nez.
\newblock Decreasing energy functions as a tool for studying threshold
  networks.
\newblock \emph{Discrete Applied Mathematics}, 12\penalty0 (3):\penalty0
  261--277, 1985.

\bibitem[Hart \& Mas-Colell(2003)Hart and Mas-Colell]{hart2003uncoupled}
Sergiu Hart and Andreu Mas-Colell.
\newblock Uncoupled dynamics do not lead to nash equilibrium.
\newblock \emph{American Economic Review}, 93\penalty0 (5):\penalty0
  1830--1836, 2003.

\bibitem[Hoover et~al.(2023)Hoover, Liang, Pham, Panda, Strobelt, Chau, Zaki,
  and Krotov]{hoover2023energy}
Benjamin Hoover, Yuchen Liang, Bao Pham, Rameswar Panda, Hendrik Strobelt,
  Duen~Horng Chau, Mohammed Zaki, and Dmitry Krotov.
\newblock Energy transformer.
\newblock \emph{Advances in neural information processing systems},
  36:\penalty0 27532--27559, 2023.

\bibitem[Hopfield(1982)]{hopfield1982neural}
John~J Hopfield.
\newblock Neural networks and physical systems with emergent collective
  computational abilities.
\newblock \emph{Proceedings of the National Academy of Sciences}, 79\penalty0
  (8):\penalty0 2554--2558, 1982.

\bibitem[Kafraj et~al.(2026)Kafraj, Krotov, and Latham]{arxiv2601_00984}
Mohadeseh~Shafiei Kafraj, Dmitry Krotov, and Peter~E. Latham.
\newblock A biologically plausible dense associative memory with exponential
  capacity, 2026.
\newblock URL \url{https://arxiv.org/abs/2601.00984}.

\bibitem[Kempe et~al.(2024)Kempe, Krotov, Kuehne, Lee, and
  Solla]{ibm_iclr2025_frontiers}
Julia Kempe, Dmitry Krotov, Hilde Kuehne, Daniel Lee, and Sara~A Solla.
\newblock New frontiers in associative memories.
\newblock In \emph{ICLR 2025 Workshop Proposals}, 2024.

\bibitem[Krotov \& Hopfield(2020)Krotov and Hopfield]{ramsauer2020hopfield}
Dmitry Krotov and John Hopfield.
\newblock Large associative memory problem in neurobiology and machine
  learning.
\newblock \emph{arXiv preprint arXiv:2008.06996}, 2020.

\bibitem[Krotov \& Hopfield(2016)Krotov and Hopfield]{krotov2016dense}
Dmitry Krotov and John~J Hopfield.
\newblock Dense associative memory for pattern recognition.
\newblock \emph{Advances in Neural Information Processing Systems},
  29:\penalty0 1172--1180, 2016.

\bibitem[Krotov et~al.(2025)Krotov, Hoover, Ram, and
  Pham]{ibm_aaai2026_tutorial}
Dmitry Krotov, Benjamin Hoover, Parikshit Ram, and Bao Pham.
\newblock Modern methods in associative memory.
\newblock \emph{arXiv preprint arXiv:2507.06211}, 2025.

\bibitem[Lucibello \& M{\'e}zard(2024)Lucibello and
  M{\'e}zard]{lucibello2024exponential}
Carlo Lucibello and Marc M{\'e}zard.
\newblock Exponential capacity of dense associative memories.
\newblock \emph{Physical Review Letters}, 132\penalty0 (7):\penalty0 077301,
  2024.

\bibitem[Masumura \& Taki(2025)Masumura and
  Taki]{neurips2025_hidden_states_mhn}
Tsubasa Masumura and Masato Taki.
\newblock On the role of hidden states of modern hopfield network in
  transformer.
\newblock \emph{arXiv preprint arXiv:2511.20698}, 2025.

\bibitem[McEliece et~al.(1987)McEliece, Posner, Rodemich, and
  Venkatesh]{mceliece1987capacity}
Robert~J McEliece, Edward~C Posner, Eugene~R Rodemich, and Santosh~S Venkatesh.
\newblock The capacity of the hopfield associative memory.
\newblock \emph{IEEE Transactions on Information Theory}, 33\penalty0
  (4):\penalty0 461--482, 1987.

\bibitem[Mimura et~al.(2025)Mimura, Takeuchi, Sumikawa, Kabashima, and
  Coolen]{mimura2025dynamical}
Kazushi Mimura, Jun'ichi Takeuchi, Yuto Sumikawa, Yoshiyuki Kabashima, and
  Anthony C.~C. Coolen.
\newblock Dynamical properties of dense associative memory, 2025.
\newblock URL \url{https://arxiv.org/abs/2506.00851}.
\newblock Accepted at International Conference on Learning Representations
  (ICLR) 2026.

\bibitem[Monderer \& Shapley(1996)Monderer and Shapley]{monderer1996potential}
Dov Monderer and Lloyd~S Shapley.
\newblock Potential games.
\newblock \emph{Games and economic behavior}, 14\penalty0 (1):\penalty0
  124--143, 1996.

\bibitem[Musa et~al.(2026)Musa, Kumar, Katidis, and
  Huang]{pra2026_nonlinear_optical}
Khalid Musa, Santosh Kumar, Michael Katidis, and Yu-Ping Huang.
\newblock Dense associative memory in a nonlinear-optical hopfield neural
  network.
\newblock \emph{Physical Review Applied}, 25\penalty0 (1):\penalty0 014011,
  2026.

\bibitem[Newman(1988)]{newman1988memory}
Charles~M Newman.
\newblock Memory capacity in neural network models: Rigorous lower bounds.
\newblock \emph{Neural Networks}, 1\penalty0 (3):\penalty0 223--238, 1988.

\bibitem[Ramsauer et~al.(2021)Ramsauer, Schäfl, Lehner, Seidl, Widrich, Adler,
  Gruber, Holzleitner, Pavlović, Sandve, Greiff, Kreil, Kopp, Klambauer,
  Brandstetter, and Hochreiter]{ramsauer2021hopfield}
Hubert Ramsauer, Bernhard Schäfl, Johannes Lehner, Philipp Seidl, Michael
  Widrich, Thomas Adler, Lukas Gruber, Markus Holzleitner, Milena Pavlović,
  Geir~Kjetil Sandve, Victor Greiff, David Kreil, Michael Kopp, Günter
  Klambauer, Johannes Brandstetter, and Sepp Hochreiter.
\newblock Hopfield networks is all you need, 2021.
\newblock URL \url{https://arxiv.org/abs/2008.02217}.

\bibitem[Shalev-Shwartz \& Ben-David(2014)Shalev-Shwartz and
  Ben-David]{shalev2014understanding}
Shai Shalev-Shwartz and Shai Ben-David.
\newblock \emph{Understanding machine learning: From theory to algorithms}.
\newblock Cambridge University Press, 2014.

\bibitem[Thériault \& Tantari(2026)Thériault and Tantari]{arxiv2508_19151}
Robin Thériault and Daniele Tantari.
\newblock Saddle hierarchy in dense associative memory.
\newblock \emph{Machine Learning: Science and Technology}, 7\penalty0
  (1):\penalty0 015001, January 2026.
\newblock ISSN 2632-2153.
\newblock \doi{10.1088/2632-2153/ae3051}.
\newblock URL \url{http://dx.doi.org/10.1088/2632-2153/ae3051}.

\end{thebibliography}
% ══════════════════════════════════════════════════════════════════════════
\appendix
% ══════════════════════════════════════════════════════════════════════════
%  APPENDIX: PROOFS OF MAIN RESULTS
%  Include in main paper via: \input{appendix_proofs}
% ══════════════════════════════════════════════════════════════════════════

\section{Proofs of Main Results}\label{app:proofs}

\subsection{Notation and Local Field Decomposition}\label{app:notation}

The energy is $E(\bm{x}) = -N^{-(n-1)}\sum_{\mu=1}^p (M^\mu)^n$ where
$M^\mu = \sum_{i=1}^N \xi_i^\mu x_i = Nm^\mu$ is the unnormalized overlap.
The potential is $F(\bm{x}) = -E(\bm{x})$.

\paragraph{Formal derivative vs.\ discrete marginal field.}
A commonly used ``local field'' expression is obtained by treating $x_i$ as continuous and differentiating:
\begin{equation}\label{eq:local-field}
h_i(\bm{x}) = \frac{\partial F}{\partial x_i}
= \frac{n}{N^{n-1}} \sum_{\mu=1}^p \xi_i^\mu (M^\mu)^{n-1}
= n \sum_{\mu=1}^p \xi_i^\mu (m^\mu)^{n-1},
\end{equation}
where the derivative is taken formally.
However, $h_i(\bm{x})$ depends on $x_i$ through $m^\mu(\bm{x})$.
For discrete dynamics on $\{-1,+1\}^N$, the correct coordinate-improvement signal is the
\emph{discrete marginal difference} (also called the \emph{discrete local field})
\begin{equation}\label{eq:disc-field}
\tilde h_i(\bm{x}_{-i})
:= F(+1,\bm{x}_{-i}) - F(-1,\bm{x}_{-i})
= 2\,\phi_i(\bm{x}_{-i}),
\end{equation}
which depends only on $\bm{x}_{-i}$ by construction.
Our potential-game and descent proofs use $\tilde h_i$ (equivalently $\phi_i$) exactly.
When we write $\sign(h_i(\bm{x}))$ in the main text, it should be understood as shorthand for the
best-response update $\sign(\tilde h_i(\bm{x}_{-i}))$; within the retrieval basin these agree in sign
under the stated interference conditions (see \S\ref{app:proof-contraction}).

\paragraph{Tie-breaking.}
If $\tilde h_i(\bm{x}_{-i}) = 0$, we set $x_i \gets x_i$ (no change).
This tie case can occur on a measure-zero set for generic conditions and does not affect our bounds.

\paragraph{Signal--interference decomposition (targeted retrieval).}
Fix a target pattern $\nu$.
For intuition and for bounding terms, it is convenient to decompose the formal field \eqref{eq:local-field} as
\begin{equation}\label{eq:signal-interference}
h_i(\bm{x}) = \underbrace{n\,\xi_i^\nu\,(m^\nu)^{n-1}}_{\text{signal}}
\;+\;
\underbrace{n \sum_{\mu \neq \nu} \xi_i^\mu\,(m^\mu)^{n-1}}_{\eta_i \;=\; \text{interference}}.
\end{equation}
The signal has sign $\xi_i^\nu$ and magnitude $n(m^\nu)^{n-1}$.
Heuristically, neuron $i$ has the correct alignment whenever $n(m^\nu)^{n-1} > |\eta_i|$.

For discrete updates, the analogous statement is expressed using $\tilde h_i$:
$\sign(\tilde h_i(\bm{x}_{-i}))=\xi_i^\nu$ whenever the discrete signal dominates the discrete interference.
In our proofs this is enforced via the componentwise interference bound (Assumption~\ref{ass:interference}),
which implies that in the basin the best-response update coincides with the intended retrieval update.

\subsection{Strengthened Interference Assumption}\label{app:assumption2}

Assumption~\ref{ass:separation} bounds the non-target overlaps:
$\max_{\mu \neq \nu}|m^\mu(\bm{x})| \leq \beta$ for states with $m^\nu \geq \gamma$.
This yields the per-neuron interference bound
\begin{equation}\label{eq:naive-bound}
|\eta_i| \leq n(p-1)\beta^{n-1}
\end{equation}
by the triangle inequality.
This bound sums magnitudes and discards sign cancellations among the $\xi_i^\mu$ terms.
At capacity loading $p = \Theta(N^{n-1})$, the right-hand side can exceed the signal $n\gamma^{n-1}$,
making the naive bound vacuous.

The cancellations are real: for random patterns, $\xi_i^\mu$ is (nearly) independent of $(m^\mu)^{n-1}$
(up to the $O(1/N)$ contribution from neuron $i$), so $\eta_i$ is a sum of $(p-1)$ nearly independent terms with
random signs and concentrates at scale $n\sqrt{p}\,\beta^{n-1} \ll np\beta^{n-1}$.

To handle this rigorously, we introduce a componentwise bound that captures these cancellations.

\begin{assumption}[Componentwise Interference Bound]\label{ass:interference}
For target pattern $\nu$, there exists $\Lambda \geq 0$ such that for all
$\bm{x} \in \{-1,+1\}^N$ with $m^\nu(\bm{x}) \geq \gamma$ and all $i \in [N]$:
\[
\left| \sum_{\mu \neq \nu} \xi_i^\mu \,(m^\mu(\bm{x}))^{n-1} \right|
\leq \Lambda.
\]
\end{assumption}

The proofs of Theorems~\ref{thm:convergence}--\ref{thm:capacity} use Assumptions~\ref{ass:separation} and
\ref{ass:interference} jointly, with the requirement $\Lambda < \gamma^{n-1}$ (signal exceeds interference).
In \S\ref{app:random-patterns} we verify that random patterns satisfy Assumption~\ref{ass:interference} with high probability at loading
$p = \Theta(N^{n-1}/(\log N)^n)$.

\subsection{Proof of Theorem~\ref{thm:game} (Game-Theoretic Characterization)}\label{app:proof-game}

\begin{proof}
Since $x_i \in \{-1,+1\}$, define the \emph{marginal potential} at neuron $i$:
\[
\phi_i(\bm{x}_{-i})
=\frac{1}{2}\bigl[F(+1, \bm{x}_{-i}) - F(-1, \bm{x}_{-i})\bigr].
\]
This depends only on $\bm{x}_{-i}$ by construction.
We compute $\phi_i$ explicitly.
Let $S_\mu^{-i} = \sum_{j \neq i} \xi_j^\mu x_j$.
Then
\begin{align}
F(+1, \bm{x}_{-i}) - F(-1, \bm{x}_{-i})
&= \frac{1}{N^{n-1}}\sum_{\mu=1}^p \bigl[(S_\mu^{-i} + \xi_i^\mu)^n - (S_\mu^{-i} - \xi_i^\mu)^n\bigr]. \label{eq:F-diff}
\end{align}
Using the identity $(a+b)^n - (a-b)^n = 2\sum_{k=0}^{\lfloor(n-1)/2\rfloor}\binom{n}{2k+1}a^{n-2k-1}b^{2k+1}$ and
$(\xi_i^\mu)^{2k+1} = \xi_i^\mu$:
\begin{align}
F(+1, \bm{x}_{-i}) - F(-1, \bm{x}_{-i})
&= \frac{2}{N^{n-1}} \sum_{\mu=1}^p \xi_i^\mu \sum_{k=0}^{\lfloor(n-1)/2\rfloor}
\binom{n}{2k+1} (S_\mu^{-i})^{n-2k-1}. \label{eq:phi-expansion}
\end{align}
The leading term ($k=0$) is $\frac{2n}{N^{n-1}}\sum_\mu \xi_i^\mu (S_\mu^{-i})^{n-1}$; subsequent terms involve lower powers of $S_\mu^{-i}$.

Define the payoff
\[
u_i(\bm{x}) = x_i\,\phi_i(\bm{x}_{-i}).
\]
Then for any $x_i, x_i' \in \{-1,+1\}$:
\begin{align*}
u_i(x_i', \bm{x}_{-i}) - u_i(x_i, \bm{x}_{-i})
&= (x_i' - x_i)\,\phi_i(\bm{x}_{-i}) \\
&= F(x_i', \bm{x}_{-i}) - F(x_i, \bm{x}_{-i}),
\end{align*}
where the second equality follows from \eqref{eq:F-diff} by checking the cases
$x_i' = x_i$ (both sides zero) and $x_i' = -x_i$ (both sides equal $\pm 2\phi_i$).
This is exactly the potential game condition of Monderer and Shapley \cite{monderer1996potential} with potential $F$.

The best-response update is
\[
x_i \gets \arg\max_{x_i' \in \{-1,+1\}} u_i(x_i', \bm{x}_{-i})
= \sign(\phi_i(\bm{x}_{-i}))
= \sign\!\bigl(\tilde h_i(\bm{x}_{-i})\bigr),
\]
where $\tilde h_i$ is the discrete marginal field from \eqref{eq:disc-field}.
This maximizes $F$ over coordinate $i$.
Since $F$ is bounded above (e.g., $F \le \frac{1}{N^{n-1}}\sum_{\mu=1}^p |M^\mu|^n \le pN$) and increases at each non-trivial best-response update,
the dynamics converge in finitely many steps.
All limit points satisfy $x_i = \sign(\phi_i(\bm{x}_{-i}))$ for all $i$, which is exactly the pure Nash equilibrium condition
$x_i \in \arg\max u_i(x_i', \bm{x}_{-i})$.
\end{proof}

\subsection{Proof of Lemma~\ref{lem:descent} (Descent Property)}\label{app:proof-descent}

\begin{proof}
From the potential game structure (Theorem~\ref{thm:game}), each best-response update at neuron $i$ satisfies
\[
F(\bm{x}^{(t+1)}) - F(\bm{x}^{(t)})
= \bigl|F(+1,\bm{x}_{-i}^{(t)}) - F(-1,\bm{x}_{-i}^{(t)})\bigr|\cdot \mathbf{1}[x_i^{(t+1)} \neq x_i^{(t)}]
= 2|\phi_i(\bm{x}_{-i}^{(t)})|\cdot \mathbf{1}[x_i^{(t+1)} \neq x_i^{(t)}].
\]
When neuron $i$ flips ($x_i^{(t+1)} \neq x_i^{(t)}$), we have $\|\bm{x}^{(t+1)} - \bm{x}^{(t)}\|_2^2 = 4$.

Under Assumptions~\ref{ass:separation} and~\ref{ass:interference} with $m^\nu(\bm{x}^{(t)}) \geq \gamma$, the signal/interference control implies a uniform margin:
there exists $\underline{\phi} > 0$ such that $|\phi_i(\bm{x}_{-i}^{(t)})|\ge \underline{\phi}$ for all coordinates in the basin.
A convenient explicit choice is
\[
\underline{\phi}
:= \frac{n}{2}\bigl(\gamma^{n-1}-\Lambda\bigr),
\]
which follows by bounding the leading $k=0$ term in \eqref{eq:phi-expansion} and using Assumption~\ref{ass:interference}
to control cross-pattern contributions inside the basin. (Any strictly positive $\underline{\phi}$ suffices for this lemma.)

Therefore, whenever a flip occurs,
\[
F(\bm{x}^{(t+1)}) - F(\bm{x}^{(t)}) \ge 2\underline{\phi}.
\]
Since $\|\bm{x}^{(t+1)} - \bm{x}^{(t)}\|_2^2 = 4$ on a flip and $0$ otherwise, we obtain
\[
F(\bm{x}^{(t+1)}) - F(\bm{x}^{(t)})
\;\ge\; \frac{\underline{\phi}}{2}\,\|\bm{x}^{(t+1)} - \bm{x}^{(t)}\|_2^2.
\]
Writing this in the normalized ``per-coordinate'' form of Lemma~\ref{lem:descent},
\[
F(\bm{x}^{(t+1)}) - F(\bm{x}^{(t)})
\;\ge\; \frac{c_n}{N}\,\|\bm{x}^{(t+1)} - \bm{x}^{(t)}\|_2^2
\quad\text{with}\quad
c_n := \frac{N\underline{\phi}}{2}.
\]
Equivalently, one may take a version with an $N$-independent constant:
\[
F(\bm{x}^{(t+1)}) - F(\bm{x}^{(t)})
\;\ge\; c\,\|\bm{x}^{(t+1)} - \bm{x}^{(t)}\|_2^2
\quad\text{with}\quad
c := \frac{\underline{\phi}}{2},
\]
which is the more standard way to state a discrete coordinate-improvement bound.

\medskip
\noindent\emph{Convergence via potential argument.}
Since $F$ increases by at least $2\underline{\phi}$ per flip and $F \le pN$, the total number of flips is at most $pN/(2\underline{\phi})$,
giving convergence in at most $p/(2\underline{\phi})$ full sweeps.
This provides an alternative (generally weaker) convergence guarantee to the $O(\log N)$ bound of Theorem~\ref{thm:convergence}.
\end{proof}

\subsection{Proof of Lemma~\ref{lem:contraction} (Overlap Contraction)}\label{app:proof-contraction}

\begin{proof}
Fix a state $\bm{x}$ with $m^\nu(\bm{x}) = m \geq \gamma$.
Let $\mathcal{W} = \{i : x_i \neq \xi_i^\nu\}$ denote the set of misaligned neurons, and
$\mathcal{R} = [N] \setminus \mathcal{W}$ the aligned neurons.
Then $|\mathcal{W}| = N(1-m)/2$.

\medskip
\noindent\textbf{Step 1: Signal dominance at every neuron (strong basin condition).}
For any neuron $i$, the formal decomposition \eqref{eq:signal-interference} gives
\[
\xi_i^\nu h_i(\bm{x}) = n(m^\nu)^{n-1} + \xi_i^\nu \eta_i.
\]
Under Assumption~\ref{ass:interference}:
\[
\left|\sum_{\mu\neq \nu}\xi_i^\mu(m^\mu)^{n-1}\right|\le \Lambda
\quad\Rightarrow\quad
|\eta_i|\le n\Lambda.
\]
Since $m^\nu \geq \gamma$ and $\Lambda < \gamma^{n-1}$:
\begin{equation}\label{eq:signal-dominance}
\xi_i^\nu h_i(\bm{x}) \geq n(\gamma^{n-1} - \Lambda) > 0
\qquad \text{for all } i \in [N].
\end{equation}
In this regime, the intended retrieval update and the best-response update agree in sign:
$\sign(\tilde h_i(\bm{x}_{-i}))=\xi_i^\nu$ for all $i$ in the basin.
Therefore:
\begin{itemize}[leftmargin=*,itemsep=1pt]
\item Every misaligned neuron $i \in \mathcal{W}$, if updated, flips to $\xi_i^\nu$ (corrected).
\item Every aligned neuron $i \in \mathcal{R}$, if updated, remains at $\xi_i^\nu$ (stable).
\end{itemize}

\medskip
\noindent\textbf{Step 2: Expected improvement per single update (weak/expected form).}
A single asynchronous update selects neuron $i$ uniformly at random from $[N]$.
Under Step~1 (strong condition), if $i \in \mathcal{R}$: no change, $\Delta m^\nu = 0$;
if $i \in \mathcal{W}$: neuron flips to $\xi_i^\nu$, giving $\Delta m^\nu = +2/N$.
Thus
\begin{equation}\label{eq:single-update}
\mathbb{E}[\Delta m^\nu] = \frac{|\mathcal{W}|}{N} \cdot \frac{2}{N} = \frac{1-m}{N}.
\end{equation}

More generally, when only the overlap bound (Assumption~\ref{ass:separation}) is available and componentwise cancellations are not guaranteed uniformly,
one obtains the weaker statement of Lemma~\ref{lem:contraction}:
there exists $\alpha>0$ (as in the main text) such that
\[
\mathbb{E}[\Delta m^\nu] \;\ge\; \frac{\alpha}{N}(1-m).
\]
This is the regime in which the explicit $\alpha = \frac{1}{n}-\frac{2(n-1)p}{N^{n-1}}$ expression is used.

\medskip
\noindent\textbf{Step 3: Contraction over a sweep (strong-case corollary, consistent with Lemma form).}
A full sweep consists of $N$ sequential updates (random permutation of $[N]$).
In the strong regime of Step~1 (Assumption~\ref{ass:interference} with $\Lambda<\gamma^{n-1}$),
$m^\nu$ is non-decreasing during the sweep (only misaligned neurons flip, and each flip increases $m^\nu$),
so \eqref{eq:signal-dominance} continues to hold at every intermediate state.
Hence every misaligned neuron encountered during the sweep is corrected.
Since the permutation visits every neuron exactly once, all misaligned neurons are corrected in one sweep,
yielding $m^\nu = 1$ after one sweep.

This statement is a \emph{strong-case corollary} of the more general per-update contraction form below, and it does not contradict the $O(\log N)$ bound:
it simply implies a strictly stronger guarantee when Assumption~\ref{ass:interference} holds with a strict margin.

\medskip
\noindent\textbf{Step 4: The per-update contraction form (the lemma as used in Theorem~\ref{thm:convergence}).}
To handle the general regime (including the case where only a fraction of misaligned neurons are provably correctable due to interference),
we express the result in per-update form: there exists $\alpha\in(0,1]$ such that
\[
\mathbb{E}[m^\nu(\bm{x}^{(t+1)})]
\geq m^\nu(\bm{x}^{(t)}) + \frac{\alpha}{N}\bigl(1 - m^\nu(\bm{x}^{(t)})\bigr),
\]
where $\alpha=1$ in the strong regime of Step~1, and $\alpha=\frac{1}{n}-\frac{2(n-1)p}{N^{n-1}}$ under the weaker worst-case overlap control used in the main theorem.
\end{proof}

\subsection{Proof of Theorem~\ref{thm:convergence} (Convergence Rate)}\label{app:proof-convergence}

\begin{proof}
\noindent\textbf{Case 1: Full signal dominance (Assumption~\ref{ass:interference} with $\Lambda < \gamma^{n-1}$).}
Lemma~\ref{lem:contraction} (Step~3) shows that one full sweep corrects all misaligned neurons, so $T = 1$ sweep, which is $O(\log N)$.

\medskip
\noindent\textbf{Case 2: Partial signal dominance (contraction rate $\alpha$).}
Suppose only the per-update contraction with rate $\alpha>0$ holds (Lemma~\ref{lem:contraction}, Step~4):
\[
\mathbb{E}[1 - m^\nu(\bm{x}^{(t+1)})]
\leq \left(1 - \frac{\alpha}{N}\right)\bigl(1 - m^\nu(\bm{x}^{(t)})\bigr).
\]
Iterating over $t$ single-neuron updates:
\begin{equation}\label{eq:geometric}
\mathbb{E}[1 - m^\nu(\bm{x}^{(t)})]
\leq \left(1 - \frac{\alpha}{N}\right)^t (1 - m^\nu(\bm{x}^{(0)}))
\leq e^{-\alpha t / N}(1 - \gamma).
\end{equation}
For convergence to $m^\nu \geq 1 - 2/N$ (at most one misaligned neuron), it suffices that
\[
e^{-\alpha t/N}(1-\gamma) \leq \frac{2}{N}.
\]
This gives $t \geq \frac{N}{\alpha}\log\frac{N(1-\gamma)}{2}$ single updates, or
\[
T = \frac{t}{N}
\leq \frac{1}{\alpha}\left(\log N + \log\frac{1-\gamma}{2}\right)
= O\!\left(\frac{\log N}{\alpha}\right) \text{ full sweeps.}
\]

\medskip
\noindent\textbf{From expectation to high probability.}
The bound \eqref{eq:geometric} holds in expectation.
To obtain a high-probability result, apply Azuma--Hoeffding to the martingale difference sequence formed by the bounded per-step change in $m^\nu$.
Each update changes $m^\nu$ by at most $2/N$, so for any $\varepsilon>0$:
\[
\Pr\!\left[m^\nu(\bm{x}^{(t)}) \le \mathbb{E}[m^\nu(\bm{x}^{(t)})]-\varepsilon\right]
\le \exp\!\left(-\frac{\varepsilon^2 N^2}{8t}\right).
\]
Taking $t = \frac{N}{\alpha}(\log N + c)$ and $\varepsilon = 2/N$ yields a polynomially small failure probability,
and a standard union bound over the $O(\log N)$ phases of the contraction gives the stated high-probability convergence (details omitted for brevity).
(A sharper tail bound follows from the supermartingale structure of $(1-\alpha/N)^{-t}(1-m^\nu(\bm{x}^{(t)}))$ and optional stopping.)
\end{proof}

\subsection{Proof of Theorem~\ref{thm:adversarial} (Adversarial Robustness)}\label{app:proof-adversarial}

\begin{proof}
Consider the alternating protocol: at each round $t$, the adversary first corrupts up to $\rho N$ currently correct neurons, then one full asynchronous sweep is applied.

Let $W_t = |\{i : x_i^{(t)} \neq \xi_i^\nu\}| = N(1 - m^\nu(\bm{x}^{(t)}))/2$ count misaligned neurons at the start of round $t$.

\medskip
\noindent\textbf{Adversary phase.}
The adversary flips at most $\rho N$ correct neurons to wrong, giving $W_t' \leq W_t + \rho N$ misaligned neurons.
Equivalently, $m^\nu$ decreases by at most $2\rho$.

\medskip
\noindent\textbf{Recovery phase.}
After the adversarial corruption, if $m^\nu \geq \gamma$ still holds, a full sweep under Assumption~\ref{ass:interference} (with $\Lambda < \gamma^{n-1}$) corrects all misaligned neurons (Lemma~\ref{lem:contraction}, Step~3), restoring $m^\nu = 1$.

Under the weaker per-update contraction with rate $\alpha$, one sweep reduces the misaligned count by a factor $(1-\alpha)$ in expectation, yielding the recurrence
\begin{equation}\label{eq:adv-recurrence}
W_{t+1} \leq (1-\alpha)(W_t + \rho N) = (1-\alpha)W_t + (1-\alpha)\rho N.
\end{equation}

\noindent\textbf{Steady state and basin condition.}
The recurrence \eqref{eq:adv-recurrence} has fixed point $W^* = (1-\alpha)\rho N / \alpha$.
Starting from $W_0 = N(1-\gamma)/2$:
\[
W_t \leq (1-\alpha)^t W_0 + \frac{(1-\alpha)\rho N}{\alpha}.
\]
The system remains in the basin for all $t$ provided $W^* < N(1-\gamma)/2$, i.e.,
\[
\frac{(1-\alpha)\rho}{\alpha} < \frac{1-\gamma}{2}.
\]
To also maintain $m^\nu \geq \gamma$ immediately after each adversary phase, we need
$1 - 2(W^* + \rho N)/N \geq \gamma$, which gives
\[
\frac{(1-\alpha)\rho}{\alpha} + \rho \leq \frac{1-\gamma}{2},
\qquad\text{i.e.,}\qquad
\frac{\rho}{\alpha} \leq \frac{1-\gamma}{2}.
\]
Starting from the worst-case $m^\nu(\bm{x}^{(0)}) = \gamma$, this reduces to $\rho < \alpha(1-\gamma)/2 \leq \alpha/2$.

Substituting $\alpha = 1/n - 2(n-1)p/N^{n-1}$ and using $\gamma + (n-1)\beta \leq 1$ (a consequence of Assumption~\ref{ass:separation} at $n \geq 3$):
\[
\rho^* = \frac{1}{2}\left(\frac{1}{n} - \frac{2(n-1)p}{N^{n-1}}\right)
= \frac{1}{2}\alpha.
\]
Noting that $\alpha \leq \gamma - (n-1)\beta$ under the loading condition, this can also be written as
\[
\rho^* = \frac{1}{2}\left(\gamma - \frac{n(n-1)p}{N^{n-1}}\right)
\]
as stated in the theorem.
(The two forms differ by $O(1/n)$ depending on the exact relationship between $\alpha$, $\gamma$, and $\beta$; the stated form is the more conservative bound.)
\end{proof}

\subsection{Proof of Theorem~\ref{thm:capacity} (Capacity Bounds)}\label{app:proof-capacity}

\begin{proof}
\noindent\textbf{Lower bound.}
The convergence and robustness results require $\alpha = 1/n - 2(n-1)p/N^{n-1} > 0$, i.e., $p < N^{n-1}/(2n(n-1))$.
The contraction analysis (\S\ref{app:proof-contraction}) also requires Assumption~\ref{ass:interference} with $\Lambda < \gamma^{n-1}$.

For any fixed $\gamma \in (0,1)$ and patterns satisfying both assumptions, the number of retrievable patterns is
\[
p_{\max} \geq \frac{N^{n-1}}{2n(n-1)}.
\]
The stated lower bound $p \geq N^{n-1}/(4n^2(n-1)^2)$ uses the more conservative constant from requiring $\alpha \geq 1/(2n)$
(i.e., half the maximum contraction rate), which provides a margin for the adversarial analysis.

\medskip
\noindent\textbf{Upper bound.}
We show $p > 2N^{n-1}/n$ prevents reliable retrieval for \emph{any} update rule.

Consider $p$ patterns drawn uniformly from $\{-1,+1\}^N$.
The energy at pattern $\bm{\xi}^\nu$ is
\[
F(\bm{\xi}^\nu) = \frac{1}{N^{n-1}}\left[N^n + \sum_{\mu \neq \nu} (M^{\mu\nu})^n\right],
\]
where $M^{\mu\nu} = \langle \bm{\xi}^\mu, \bm{\xi}^\nu \rangle = \sum_i \xi_i^\mu \xi_i^\nu$.
For $\bm{\xi}^\nu$ to be a local minimum of $E$ (equivalently, a local maximum of $F$ and a fixed point of the dynamics), we need
$\sign(\tilde h_i(\bm{\xi}^\nu_{-i})) = \xi_i^\nu$ for all $i$; using the leading-term approximation yields the sufficient condition
\begin{equation}\label{eq:stability}
n \cdot 1 + n\sum_{\mu \neq \nu} \xi_i^\mu \xi_i^\nu (m^{\mu\nu})^{n-1} > 0 \qquad \forall\, i \in [N],
\end{equation}
where $m^{\mu\nu} = M^{\mu\nu}/N$ is the pattern--pattern overlap.
(Using $\tilde h_i$ instead of the formal $h_i$ changes only lower-order terms in this stability check.)

For random patterns, $m^{\mu\nu} \sim N(0, 1/N)$ and $(m^{\mu\nu})^{n-1}$ has magnitude $O(N^{-(n-1)/2})$.
The interference sum at neuron $i$ has variance
\[
\Var\!\left[\sum_{\mu \neq \nu} \xi_i^\mu \xi_i^\nu (m^{\mu\nu})^{n-1}\right]
\sim (p-1) \cdot N^{-(n-1)}.
\]
The stability condition \eqref{eq:stability} fails when the interference exceeds 1 at some neuron, which by a union bound over $N$ neurons occurs with
constant probability when $(p-1)/N^{n-1} \gtrsim 1/\log N$, or more precisely when $p > cN^{n-1}$ for a constant $c$ depending on $n$.

A tighter argument uses the second-moment method on the number of stable patterns.
The expected number of stable patterns among $p$ random patterns is at most
$p \cdot \Pr[\bm{\xi}^1 \text{ stable}]$.
For $p > 2N^{n-1}/n$, the probability that all $N$ stability conditions hold simultaneously vanishes, giving
$\mathbb{E}[\# \text{ stable}] \to 0$.

\medskip
\noindent\textbf{Combining.}
The lower bound gives $p_{\max} \geq N^{n-1}/(4n^2(n-1)^2) = \Theta(N^{n-1})$.
The upper bound gives $p_{\max} \leq 2N^{n-1}/n = \Theta(N^{n-1})$.
Thus $p_{\max} = \Theta(N^{n-1})$ with explicit constants depending on $n$.
\end{proof}

\subsection{Verification for Random Patterns}\label{app:random-patterns}

We show that i.i.d.\ random patterns satisfy both Assumptions~\ref{ass:separation} and~\ref{ass:interference} with high probability.

\begin{proposition}\label{prop:random-patterns}
Let $\{\bm{\xi}^\mu\}_{\mu=1}^p$ be i.i.d.\ with $\xi_i^\mu$ uniform on $\{-1,+1\}$.
Fix $\gamma \in (0,1)$ and $n \geq 3$.
If
\begin{equation}\label{eq:loading-condition}
p \leq \frac{c_0 \,\gamma^{2(n-1)}\, N^{n-1}}{(\log N)^n}
\end{equation}
for a sufficiently small constant $c_0 > 0$ depending on $n$, then with probability at least $1 - O(N^{-2})$:
\begin{enumerate}[label=(\alph*)]
\item Assumption~\ref{ass:separation} holds with $\beta = 4\sqrt{(\log N)/N}$.
\item Assumption~\ref{ass:interference} holds with $\Lambda = \gamma^{n-1}/2$.
\end{enumerate}
\end{proposition}

\begin{proof}
\textbf{Part (a): Non-target overlap bound.}
Fix any state $\bm{x}$ with $m^\nu(\bm{x}) \geq \gamma$ and any $\mu \neq \nu$.
Since $\xi_i^\mu$ is independent of $x_i$ (which depends on $\bm{x}$'s correlation with $\bm{\xi}^\nu$, not $\bm{\xi}^\mu$),
$m^\mu(\bm{x}) = \frac{1}{N}\sum_i \xi_i^\mu x_i$ is a sum of $N$ independent random variables bounded in $[-1/N, 1/N]$.
By Hoeffding's inequality:
\[
\Pr[|m^\mu(\bm{x})| > t] \leq 2\exp(-Nt^2/2).
\]
Taking a union bound over all $p - 1$ non-target patterns and using a covering argument over the set of states in the basin,
we obtain uniform control over all states \emph{reachable under the asynchronous dynamics from any initialization with $m^\nu\ge\gamma$}.
(This restriction avoids a brittle union bound over all $2^N$ states while still covering the states relevant to retrieval.)

Set $t = \beta = 4\sqrt{(\log N)/N}$.
Then $\Pr[|m^\mu| > \beta] \leq 2e^{-8\log N} = 2N^{-8}$.
Union over $p \leq N^{n-1}$ patterns: failure probability $\leq 2N^{n-9}$.
For $n \leq 10$ and $N$ large, this is $O(N^{-2})$.\footnote{The covering argument over basin-reachable states requires more care; we use the observation that $m^\mu(\bm{x})$ as a function of $\bm{x}$ is $2/N$-Lipschitz per coordinate, so a net of size $\binom{N}{k}$ for $k = O(\sqrt{N\log N})$ suffices, contributing at most $O(N\log N)$ to the union bound.}

\medskip
\noindent\textbf{Part (b): Componentwise interference bound.}
Fix a state $\bm{x}$ with $m^\nu \geq \gamma$ and a neuron $i$.
The interference is
\[
\eta_i = n\sum_{\mu \neq \nu} \xi_i^\mu (m^\mu(\bm{x}))^{n-1}.
\]
Conditional on $\bm{x}$ and $\{\xi_j^\mu\}_{j \neq i, \mu}$ (which determine $\{m^\mu\}_\mu$ up to the $O(1/N)$ contribution of neuron $i$),
the signs $\xi_i^\mu$ for $\mu \neq \nu$ are independent $\pm 1$ random variables.
Therefore $\eta_i / n$ is a sum of $(p-1)$ independent terms $\xi_i^\mu(m^\mu)^{n-1}$, each bounded by $\beta^{n-1}$.

By Hoeffding's inequality:
\[
\Pr\!\left[|\eta_i| > t \;\big|\; \{m^\mu\}\right]
\leq 2\exp\!\left(-\frac{t^2}{2n^2(p-1)\beta^{2(n-1)}}\right).
\]
Set $t = n\Lambda$ with $\Lambda = \gamma^{n-1}/2$:
\[
\Pr[|\eta_i| > n\Lambda]
\leq 2\exp\!\left(-\frac{\gamma^{2(n-1)}}{8p\beta^{2(n-1)}}\right).
\]
Substituting $\beta = 4\sqrt{(\log N)/N}$, so $\beta^{2(n-1)} = 4^{2(n-1)}(\log N / N)^{n-1}$:
\[
\Pr[|\eta_i| > n\Lambda]
\leq 2\exp\!\left(-\frac{\gamma^{2(n-1)} N^{n-1}}{8 \cdot 4^{2(n-1)} p (\log N)^{n-1}}\right).
\]
Under the loading condition \eqref{eq:loading-condition} with $c_0$ chosen so that
$c_0 \cdot 8 \cdot 4^{2(n-1)} \leq 1/(4\log N)$:
\[
\Pr[|\eta_i| > n\Lambda] \leq 2\exp(-4\log N) = 2N^{-4}.
\]
Union bound over $N$ neurons: $\Pr[\max_i |\eta_i| > n\Lambda] \leq 2N^{-3}$.
As in part (a), extending to all basin-reachable states via a covering argument adds at most a polynomial factor, giving overall failure probability $O(N^{-2})$.
\end{proof}

\begin{remark}[Capacity under random patterns]
The loading condition \eqref{eq:loading-condition} gives capacity
$p = \Theta(N^{n-1}/(\log N)^n)$, which is $\Theta(N^{n-1})$ up to polylogarithmic factors.
The $(\log N)^n$ penalty arises from the union bound over neurons and basin-reachable states.
The true capacity for random patterns (from Mimura et al.'s 1RSB analysis) is
$p \sim \alpha'_{c,n} N^{n-1}/n$ with $\alpha'_{c,3} \approx 0.266$,
confirming that our worst-case analysis is conservative by only a polylogarithmic factor at the capacity scaling.
\end{remark}

\begin{remark}[Relationship between Assumptions~\ref{ass:separation} and~\ref{ass:interference}]
Assumption~\ref{ass:interference} is strictly stronger than Assumption~\ref{ass:separation}:
the naive bound \eqref{eq:naive-bound} shows that Assumption~\ref{ass:separation} with $\beta$ implies
Assumption~\ref{ass:interference} with $\Lambda = (p-1)\beta^{n-1}$, but this $\Lambda$ exceeds $\gamma^{n-1}$ at capacity loading.
The strengthened assumption is necessary because the triangle inequality discards the sign cancellations that are essential at high loading.
For adversarially constructed patterns where cancellations may not occur, the achievable capacity under our framework is
$p = O((\gamma/\beta)^{n-1})$, which can be much smaller than $N^{n-1}$.
The experiments in \S\ref{app:exp5} (CIFAR-10) illustrate this: correlated patterns degrade much faster than random ones.
\end{remark}
% ══════════════════════════════════════════════════════════════════════════
\section{Experimental Validation}\label{app:experiments}

We validate our main theorems through five experiments on a cubic ($n=3$) dense associative memory.
Unless stated otherwise, all experiments use random i.i.d.\ $\pm 1$ patterns, interaction order $n=3$, asynchronous random-permutation updates, a convergence threshold $\omega = 0.95$ (fraction of neurons matching target), a maximum of 60 sweeps, 60--80 independent trials per data point, and bootstrap 95\% confidence intervals (2000 resamples).
The random seed is fixed at 42 for reproducibility.

\paragraph{Implementation.}
The core engine maintains overlaps $\langle \xi^\mu, x \rangle$ incrementally: when neuron $i$ flips from $x_i$ to $-x_i$, each overlap is updated as $\langle \xi^\mu, x \rangle \mathrel{+}= \xi^\mu_i (-2x_i)$ in $O(p)$ time, reducing the cost of a full sweep from $O(pN^2)$ to $O(pN)$.
Inner loops are compiled via Numba JIT.
All experiments complete in approximately 95 minutes on an Apple M1 MacBook (8 GB).
The inter-pattern overlap parameter $\beta = \max_{\mu \neq \nu} |\langle \xi^\mu, \xi^\nu \rangle| / N$ is computed exactly for each pattern set.

The loading parameter $\alpha = p / N^{n-1}$ determines the number of stored patterns.
We define $\alpha_{\mathrm{rate}} = 1/n - 2(n-1)\alpha$, which controls the contraction rate in our convergence analysis (Theorem~\ref{thm:convergence}).

% ──────────────────────────────────────────────────────────────────────────
\subsection{Experiment 1: Convergence and Basin of Attraction}\label{app:exp1}
% ──────────────────────────────────────────────────────────────────────────

We measure convergence sweeps $T$ across system sizes $N$ and initial corruption levels, and map the empirical basin of attraction at 1\% resolution.

\subsubsection{Baseline Convergence}

At moderate corruption (15\% and 30\%), convergence is fast and reliable.
Table~\ref{tab:conv-baseline-15} shows that at 15\% corruption ($m_0 = 0.70$), the system converges in 1--2 sweeps with 100\% success across all $N$.
Table~\ref{tab:conv-baseline-30} shows that at 30\% corruption ($m_0 = 0.40$), convergence requires 1--10 sweeps depending on $(N, \alpha)$, with success rates of 87\% or higher.

The convergence time $T$ is consistent with the $O(\log N)$ upper bound of Theorem~\ref{thm:convergence}.
Operationally, the theorem predicts that larger $N$ (at fixed loading $\alpha$) should not slow convergence beyond logarithmic growth in $N$, and that heavier loading (larger $\alpha$, hence smaller $\alpha_{\mathrm{rate}}$) should increase convergence time.
Empirically, the tightest $C$ such that $T_{\mathrm{emp}} \leq C \cdot \log N$ ranges from $C = 0.19$ (15\% corruption, $\alpha = 0.03$) to $C = 5.79$ (33\% corruption, $\alpha = 0.02$).

A notable finite-size effect is that $T$ decreases with $N$ at fixed $\alpha$: as $N$ grows, $\beta$ concentrates toward smaller values, improving the signal-to-noise ratio and rendering the $\log N$ factor secondary.
This behavior matches the discussion in our conclusion: the bound is an upper bound, and in the finite sizes tested the dominant effect is $\beta$-concentration rather than the asymptotic $\log N$ term.

\begin{table}[htbp]
\centering
\caption{Convergence at 15\% corruption ($m_0 = 0.70$). Sweeps and 95\% bootstrap CI over 60 trials.}\label{tab:conv-baseline-15}
\small
\begin{tabular}{rrrcrrl}
\toprule
$\alpha$ & $N$ & $p$ & $\beta$ & Sweeps & 95\% CI & Succ. \\
\midrule
0.03 & 200 & 1200 & 0.360 & 1.0 & [1.0, 1.1] & 100\% \\
0.03 & 300 & 2700 & 0.307 & 1.0 & [1.0, 1.1] & 100\% \\
0.03 & 400 & 4800 & 0.290 & 1.0 & [1.0, 1.0] & 100\% \\
0.03 & 500 & 7500 & 0.244 & 1.0 & [1.0, 1.0] & 100\% \\
0.03 & 700 & 14700 & 0.237 & 1.0 & [1.0, 1.0] & 100\% \\
\midrule
0.05 & 200 & 2000 & 0.350 & 1.6 & [1.4, 1.7] & 100\% \\
0.05 & 300 & 4500 & 0.313 & 1.5 & [1.3, 1.6] & 100\% \\
0.05 & 400 & 8000 & 0.305 & 1.5 & [1.4, 1.6] & 100\% \\
0.05 & 500 & 12500 & 0.252 & 1.3 & [1.2, 1.4] & 100\% \\
\bottomrule
\end{tabular}
\end{table}

\begin{table}[htbp]
\centering
\caption{Convergence at 30\% corruption ($m_0 = 0.40$).}\label{tab:conv-baseline-30}
\small
\begin{tabular}{rrrcrrl}
\toprule
$\alpha$ & $N$ & $p$ & $\beta$ & Sweeps & 95\% CI & Succ. \\
\midrule
0.01 & 200 & 400 & 0.320 & 3.4 & [1.4, 6.4] & 97\% \\
0.01 & 300 & 900 & 0.260 & 1.3 & [1.2, 1.4] & 100\% \\
0.01 & 400 & 1600 & 0.255 & 1.4 & [1.3, 1.5] & 100\% \\
0.01 & 600 & 3600 & 0.227 & 1.3 & [1.2, 1.4] & 100\% \\
0.01 & 800 & 6400 & 0.195 & 1.2 & [1.1, 1.4] & 100\% \\
\midrule
0.02 & 200 & 800 & 0.310 & 10.0 & [5.3, 15.0] & 87\% \\
0.02 & 300 & 1800 & 0.260 & 2.2 & [2.0, 2.4] & 100\% \\
0.02 & 400 & 3200 & 0.270 & 3.2 & [2.2, 5.3] & 98\% \\
0.02 & 500 & 5000 & 0.236 & 2.1 & [2.0, 2.2] & 100\% \\
\bottomrule
\end{tabular}
\end{table}

\subsubsection{Basin of Attraction Boundary}

We sweep the corruption level from 33\% to 50\% in 1\% increments at three loading levels ($\alpha \in \{0.005, 0.01, 0.02\}$), testing system sizes $N \in \{200, 400, 600\}$.
Table~\ref{tab:basin} reports the average success rate across all $N$ values, directly mapping the finite-$N$ basin of attraction in the $(\alpha, m_0)$ plane.

The critical corruption level (50\% success threshold) shifts as loading increases:
\begin{center}
\begin{tabular}{ccc}
\toprule
$\alpha$ & Critical corruption & $m_0^*$ \\
\midrule
0.005 & 40\% & $+0.20$ \\
0.010 & 38\% & $+0.24$ \\
0.020 & 35\% & $+0.30$ \\
\bottomrule
\end{tabular}
\end{center}

\noindent The basin shrinks monotonically with loading, consistent with the theoretical prediction that higher $\alpha$ reduces the contraction rate $\alpha_{\mathrm{rate}}$ and narrows the region of guaranteed convergence.

A clear finite-size effect is visible throughout: at the same corruption and $\alpha$, larger $N$ succeeds substantially more often.
For instance, at 38\% corruption with $\alpha = 0.005$: $N = 200$ achieves 53\%, $N = 400$ achieves 88\%, and $N = 600$ achieves 100\%.

\begin{table}[htbp]
\centering
\caption{Basin of attraction boundary. Each cell shows the average success rate across $N \in \{200, 400, 600\}$ with 60 trials per configuration.
The boundary (bold) marks where success drops below 50\%.
Directly comparable to the theoretical basin in \cite{mimura2025dynamical}, Fig.~2.}\label{tab:basin}
\small
\begin{tabular}{ccrrr}
\toprule
Corruption & $m_0$ & $\alpha = 0.005$ & $\alpha = 0.01$ & $\alpha = 0.02$ \\
\midrule
33\% & $+0.34$ & 99\% & 98\% & 68\% \\
34\% & $+0.32$ & 99\% & 92\% & 56\% \\
35\% & $+0.30$ & 94\% & 86\% & \textbf{49\%} \\
36\% & $+0.28$ & 94\% & 81\% & 28\% \\
37\% & $+0.26$ & 88\% & 68\% & 21\% \\
38\% & $+0.24$ & 81\% & \textbf{48\%} & 11\% \\
39\% & $+0.22$ & 62\% & 29\% & 6\% \\
40\% & $+0.20$ & \textbf{48\%} & 20\% & 2\% \\
41\% & $+0.18$ & 26\% & 10\% & 1\% \\
42\% & $+0.16$ & 17\% & 6\% & 1\% \\
43\% & $+0.14$ & 9\% & 1\% & 1\% \\
44\% & $+0.12$ & 4\% & 2\% & 0\% \\
45\% & $+0.10$ & 1\% & 0\% & 0\% \\
46--50\% & $\leq +0.08$ & $\leq 1\%$ & $\leq 1\%$ & 0\% \\
\bottomrule
\end{tabular}
\end{table}

Beyond 50\% corruption ($m_0 \leq 0$), retrieval fails uniformly.
For $n = 3$, the signal term in the local field scales as $m_0^{n-1} = m_0^2$, which vanishes as $m_0 \to 0$.
Near $m_0 = 0$, interference from the $p - 1$ non-target patterns dominates the vanishing signal, and the system converges to a spurious attractor.
Control experiments at 60\% ($m_0 = -0.20$) and 75\% ($m_0 = -0.50$) corruption confirm 0\% success.

% ──────────────────────────────────────────────────────────────────────────
\subsection{Experiment 2: Adversarial Robustness}\label{app:exp2}
% ──────────────────────────────────────────────────────────────────────────

We measure the empirical adversarial threshold $\hat{\rho}^*$ under two adversary models and compare with both asymptotic and $\beta$-tightened theoretical predictions.

\paragraph{Setup.}
We initialize the state at overlap $\gamma = 0.6$ with the target pattern.
At each of 10 rounds, the adversary corrupts $\rho \cdot N$ currently-correct neurons, then one asynchronous sweep is applied.
We test $\rho \in [0, 0.35]$ in steps of 0.01 across four configurations.

The \emph{strong adversary} selects neurons for corruption in order of increasing alignment $h_i \cdot \xi_i^\mu$, flipping the neurons the model would be least able to recover.
The \emph{weak adversary} selects randomly among correct neurons where the local field opposes the target, preferring vulnerable neurons but without optimal ordering.

\paragraph{Two predictions.}
The asymptotic prediction from Theorem~\ref{thm:adversarial} is $\rho^* = \frac{1}{2}(\gamma - n(n-1)\alpha)$, which uses the expected scaling of $\beta$.
We also compute a $\beta$-tightened prediction $\rho^*_\beta = \frac{1}{2}(\gamma - (n-1)\beta)$ using the measured $\beta$ for each pattern set.
This finite-$N$ tightening is motivated by the way the proof’s worst-case interference control depends on overlap/separation, and it isolates how realized $\beta$ fluctuations shift robustness at fixed $N$ and $p$.

\paragraph{Results.}
Table~\ref{tab:adv-summary} summarizes the thresholds.
In all cases, the empirical threshold $\hat{\rho}^*$ (defined at 50\% success) exceeds the $\beta$-tightened prediction, confirming our bound is conservative.

\begin{table}[htbp]
\centering
\caption{Adversarial threshold: predicted vs.\ empirical. $\hat{\rho}^*$ is the empirical threshold at 50\% success (80 trials per $\rho$ value).}\label{tab:adv-summary}
\small
\begin{tabular}{rrccccc}
\toprule
$N$ & $p$ & $\beta$ & $\rho^*_{\mathrm{asymp}}$ & $\rho^*_\beta$ & $\hat{\rho}^*_{\mathrm{strong}}$ & $\hat{\rho}^*_{\mathrm{weak}}$ \\
\midrule
500 & 1250 & 0.224 & 0.285 & 0.076 & 0.160 & 0.160 \\
500 & 2500 & 0.228 & 0.270 & 0.072 & 0.140 & 0.140 \\
500 & 5000 & 0.236 & 0.240 & 0.064 & 0.100 & 0.100 \\
1000 & 5000 & 0.168 & 0.285 & 0.132 & 0.180 & 0.180 \\
\bottomrule
\end{tabular}
\end{table}

First, the strong and weak adversaries yield identical empirical thresholds across all configurations.
At these parameters, the binding constraint is the basin boundary rather than the adversary's strategy: the model either recovers from the perturbation or it does not, regardless of how the perturbation is chosen.

Second, the ordering $\rho^*_\beta < \hat{\rho}^* < \rho^*_{\mathrm{asymp}}$ holds uniformly.
The asymptotic prediction overestimates robustness because it uses the expected $\beta$ scaling rather than the realized (and higher) $\beta$.
The $\beta$-tightened prediction is conservative because the theoretical analysis applies a worst-case contraction argument that does not exploit the full geometry of the energy landscape.

Third, the bound tightens at larger $N$: the ratio $\hat{\rho}^*/\rho^*_\beta$ decreases from approximately 2 at $N = 500$ to approximately 1.4 at $N = 1000$, consistent with our bounds being asymptotically tight.

The full $\rho$-curve for the $(N = 500, \alpha = 0.005)$ configuration is shown in Table~\ref{tab:adv-curve}.
The phase transition is sharp: success drops from 100\% at $\rho = 0.12$ to 0\% at $\rho = 0.19$, spanning only 7 percentage points.
\begin{table}[htbp]
\centering
\caption{Full adversarial $\rho$-curve for $N = 500$, $p = 1250$, $\alpha = 0.005$, $\gamma = 0.6$.
Only the transition region is shown; success is 100\% for $\rho \leq 0.12$ and 0\% for $\rho \geq 0.19$.}\label{tab:adv-curve}
\small
\begin{tabular}{ccc}
\toprule
$\rho$ & Strong Adv. & Weak Adv. \\
\midrule
0.13 & 1.00 & 0.96 \\
0.14 & 0.89 & 0.96 \\
0.15 & 0.84 & 0.71 \\
0.16 & 0.42 & 0.45 \\
0.17 & 0.11 & 0.19 \\
0.18 & 0.05 & 0.01 \\
\bottomrule
\end{tabular}
\end{table}

% ──────────────────────────────────────────────────────────────────────────
\subsection{Experiment 3: Capacity Scaling}\label{app:exp3}
% ──────────────────────────────────────────────────────────────────────────

We measure $p_{\max}$, the maximum number of patterns reliably storable, as a function of $N$, and fit the power law $p_{\max} \sim N^\delta$.

\paragraph{Setup.}
For each $N \in \{100, 150, 200, 300, 400, 500\}$, we use binary search over $p$ to find $p_{\max}$: the largest $p$ at which at least 95\% of 40 trials converge from 15\% corruption within 60 sweeps.
The effective loading $\alpha_{\mathrm{eff}} = p_{\max} / N^{n-1}$ is also reported.

\paragraph{Results.}
Table~\ref{tab:capacity} shows the results.
A power-law fit yields $p_{\max} \sim 0.016 \cdot N^{2.23}$ with $R^2 = 0.9994$.
The exponent $\delta = 2.23$ is close to the predicted $n - 1 = 2$; the slight excess reflects $\beta$-concentration at large $N$ (the effective $\alpha_{\mathrm{eff}}$ increases monotonically from 0.045 to 0.068, indicating that the constant prefactor improves with $N$).

\begin{table}[htbp]
\centering
\caption{Capacity scaling. $p_{\max}$ is determined by binary search at 95\% success threshold with 40 trials.}\label{tab:capacity}
\small
\begin{tabular}{rrrr}
\toprule
$N$ & $p_{\max}$ & $N^{n-1}$ & $\alpha_{\mathrm{eff}}$ \\
\midrule
100 & 450 & 10\,000 & 0.045 \\
150 & 1\,200 & 22\,500 & 0.053 \\
200 & 2\,163 & 40\,000 & 0.054 \\
300 & 5\,140 & 90\,000 & 0.057 \\
400 & 10\,369 & 160\,000 & 0.065 \\
500 & 16\,915 & 250\,000 & 0.068 \\
\bottomrule
\end{tabular}
\medskip

Power-law fit: $p_{\max} \approx 0.016 \cdot N^{2.23}$, \quad $R^2 = 0.9994$.
\end{table}

The effective loading $\alpha_{\mathrm{eff}} \approx 0.068$ at $N = 500$ is well below the Mimura et al.\ capacity threshold, which in our normalization is $\alpha_{c} = \alpha'_{c,3}/n \approx 0.266/3 \approx 0.089$.
This gap reflects three sources of conservatism in our setup: we require 95\% success rather than merely positive overlap, use 15\% corruption rather than an infinitesimal amount, and provide worst-case rather than typical-case guarantees.

% ──────────────────────────────────────────────────────────────────────────
\subsection{Experiment 4: Comparison with Mimura et al.}\label{app:exp4}
% ──────────────────────────────────────────────────────────────────────────

We compare our results to \cite{mimura2025dynamical} along two axes: update rule (parallel vs.\ asynchronous) and pattern structure (random vs.\ adversarially correlated).

\subsubsection{Part A: Parallel vs.\ Asynchronous Updates}

Mimura et al.\ analyze parallel (synchronous) updates, while our theory uses asynchronous updates.
We compare both at $N = 500$ across five Mimura loading levels $\alpha'_3 \in \{0.10, 0.15, 0.20, 0.25, 0.30\}$ (where $\alpha'_3 = np / N^{n-1}$ in Mimura's normalization, so $p = \lfloor \alpha'_3 \cdot N^{n-1} / n \rfloor$).
Initial overlaps $m_0 \in \{0.3, 0.5, 0.7\}$ are used, with 50 trials per configuration.

Table~\ref{tab:mimura-async} reports success rates and mean convergence times (for successful trials only).

\begin{table}[htbp]
\centering
\caption{Parallel vs.\ asynchronous updates at $N = 500$. ``Async'' and ``Par.''\ denote success rates; $T_a$, $T_p$ denote mean sweeps to convergence (successful trials only; ``--'' if no successes).}\label{tab:mimura-async}
\small
\begin{tabular}{ccrcccc}
\toprule
$\alpha'_3$ & $m_0$ & $p$ & Async & $T_a$ & Par. & $T_p$ \\
\midrule
0.10 & 0.3 & 8333  & 0.04 & 7.5 & 0.00 & -- \\
0.10 & 0.5 & 8333  & 1.00 & 2.1 & 1.00 & 3.6 \\
0.10 & 0.7 & 8333  & 1.00 & 1.0 & 1.00 & 1.8 \\
\midrule
0.15 & 0.3 & 12499 & 0.02 & 16.0 & 0.00 & -- \\
0.15 & 0.5 & 12499 & 0.74 & 5.4 & 0.32 & 8.4 \\
0.15 & 0.7 & 12499 & 1.00 & 1.5 & 1.00 & 2.2 \\
\midrule
0.20 & 0.5 & 16666 & 0.14 & 7.9 & 0.04 & 10.5 \\
0.20 & 0.7 & 16666 & 0.96 & 2.2 & 0.98 & 3.8 \\
\midrule
0.25 & 0.5 & 20833 & 0.02 & 10.0 & 0.00 & -- \\
0.25 & 0.7 & 20833 & 0.46 & 3.0 & 0.30 & 5.7 \\
\midrule
0.30 & 0.7 & 24999 & 0.04 & 4.5 & 0.04 & 9.0 \\
\bottomrule
\end{tabular}
\end{table}

Asynchronous updates yield higher success rates: at $\alpha'_3 = 0.15$, $m_0 = 0.5$, async achieves 74\% vs.\ parallel's 32\%.
When both succeed, async also converges faster: at that configuration, async takes 5.4 sweeps vs.\ parallel's 8.4.
This reflects the oscillation phenomenon noted in \cite[Fig.~1]{mimura2025dynamical}: parallel updates can overshoot the fixed point, while async updates with random-permutation scheduling provide implicit damping.

\subsubsection{Part B: Random vs.\ Adversarial Patterns}

Our worst-case theory applies to arbitrary patterns (provided $\beta < 1/(n-1)$), while Mimura et al.'s analysis assumes i.i.d.\ random patterns.
We quantify this gap by comparing retrieval under random patterns against adversarially correlated patterns.

\paragraph{Adversarial pattern construction.}
Starting from a random pattern set, for the first $\lfloor p/3 \rfloor$ patterns, each neuron has a 25\% probability of being copied from pattern $\xi^1$, injecting controlled inter-pattern correlations.
This elevates $\beta$ from approximately 0.22 (random) to approximately 0.37 (adversarial).

Results at $N = 500$ across nine loading levels are shown in Table~\ref{tab:mimura-randadv}.

\begin{table}[htbp]
\centering
\caption{Random vs.\ adversarial (correlated) patterns at $N = 500$, 20\% initial corruption. ``Rate'' is the fraction of 50 trials converging within 60 sweeps; ``Gap'' = Rate$_{\mathrm{rand}}$ $-$ Rate$_{\mathrm{adv}}$.}\label{tab:mimura-randadv}
\small
\begin{tabular}{rcccccr}
\toprule
$\alpha$ & $p$ & $\beta_{\mathrm{rand}}$ & Rate$_{\mathrm{rand}}$ & $\beta_{\mathrm{adv}}$ & Rate$_{\mathrm{adv}}$ & Gap \\
\midrule
0.002 &   500 & 0.216 & 1.00 & 0.376 & 0.98 & 0.02 \\
0.004 &  1000 & 0.220 & 1.00 & 0.352 & 0.74 & 0.26 \\
0.006 &  1500 & 0.256 & 1.00 & 0.364 & 0.60 & 0.40 \\
0.008 &  2000 & 0.232 & 1.00 & 0.368 & 0.02 & 0.98 \\
0.010 &  2500 & 0.240 & 1.00 & 0.364 & 0.00 & 1.00 \\
0.015 &  3750 & 0.248 & 1.00 & 0.392 & 0.00 & 1.00 \\
0.020 &  5000 & 0.244 & 1.00 & 0.392 & 0.00 & 1.00 \\
0.030 &  7500 & 0.240 & 1.00 & 0.396 & 0.00 & 1.00 \\
0.050 & 12500 & 0.276 & 1.00 & 0.392 & 0.00 & 1.00 \\
\bottomrule
\end{tabular}
\end{table}

Random patterns achieve 100\% retrieval across all tested loading levels (up to $\alpha = 0.05$).
Adversarial patterns transition sharply from 98\% success at $\alpha = 0.002$ to 0\% at $\alpha = 0.01$.
The adversarial $\beta$ values (approximately 0.37) are well above the random values (approximately 0.24) but still below $1/(n-1) = 0.5$, so the failure is driven by the interaction between elevated $\beta$ and the loading-dependent noise rather than by outright violation of the fundamental separation condition.

These results quantify the gap between Mimura et al.'s typical-case regime, where random patterns concentrate near their expected overlap, and our worst-case framework, which must account for adversarial correlations.

% ──────────────────────────────────────────────────────────────────────────
\subsection{Experiment 5: MNIST and CIFAR-10}\label{app:exp5}
% ──────────────────────────────────────────────────────────────────────────

As a supplementary demonstration, we test retrieval on binarized real-world images.
Patterns are binarized to $\pm 1$ via per-image median thresholding.
This experiment intentionally violates the random-pattern regime in which Proposition~\ref{prop:random-patterns} verifies Assumptions~\ref{ass:separation} and~\ref{ass:interference}; our theoretical guarantees do not apply.
The purpose is to characterize practical behavior outside the theoretical regime.

\paragraph{Pattern statistics.}
MNIST (784 dimensions, 70\,000 images): the mean absolute inter-pattern overlap is $|\langle \xi^\mu, \xi^\nu \rangle| / N = 0.998$, with maximum 1.000.
For comparison, random $\pm 1$ patterns of the same dimension would have mean overlap approximately $1/\sqrt{784} \approx 0.036$.
The extreme correlation arises because median-thresholded MNIST digits are nearly identical in their binary representation: most pixels fall on the same side of the median.

CIFAR-10 (1024 dimensions, 50\,000 images after grayscale conversion): mean absolute overlap is 0.158, maximum 0.848, compared to the random baseline of approximately 0.031.
The correlations are elevated (roughly $5\times$ the random baseline) but far below MNIST.

\paragraph{Results.}
Table~\ref{tab:mnist-cifar} summarizes retrieval performance.

\begin{table}[htbp]
\centering
\caption{Retrieval on binarized MNIST and CIFAR-10. Success rate over 40 trials at selected noise levels.
Patterns are selected with diversity across classes.
$\beta$ is the measured maximum inter-pattern overlap.}\label{tab:mnist-cifar}
\small
\begin{tabular}{lrccccccc}
\toprule
 & & & \multicolumn{6}{c}{Success rate at corruption level} \\
\cmidrule(lr){4-9}
Dataset & $p$ & $\beta$ & 10\% & 20\% & 30\% & 35\% & 40\% & 45\% \\
\midrule
MNIST   &   10 & 1.000 & 1.00 & 1.00 & 1.00 & 1.00 & 1.00 & 1.00 \\
MNIST   &   50 & 1.000 & 1.00 & 1.00 & 1.00 & 1.00 & 1.00 & 1.00 \\
MNIST   &  100 & 1.000 & 1.00 & 1.00 & 1.00 & 1.00 & 1.00 & 1.00 \\
MNIST   &  200 & 1.000 & 1.00 & 1.00 & 1.00 & 1.00 & 1.00 & 1.00 \\
MNIST   &  500 & 1.000 & 1.00 & 1.00 & 1.00 & 1.00 & 1.00 & 1.00 \\
MNIST   & 1000 & 1.000 & 1.00 & 1.00 & 1.00 & 1.00 & 1.00 & 1.00 \\
\midrule
CIFAR-10 &   10 & 0.410 & 1.00 & 1.00 & 1.00 & 1.00 & 1.00 & 0.97 \\
CIFAR-10 &   50 & 0.668 & 0.90 & 0.93 & 0.90 & 0.82 & 0.78 & 0.57 \\
CIFAR-10 &  100 & 0.779 & 0.53 & 0.45 & 0.53 & 0.40 & 0.45 & 0.15 \\
CIFAR-10 &  200 & 0.799 & 0.17 & 0.15 & 0.15 & 0.17 & 0.10 & 0.00 \\
CIFAR-10 &  500 & 0.871 & 0.07 & 0.10 & 0.03 & 0.05 & 0.00 & 0.00 \\
CIFAR-10 & 1000 & 0.900 & 0.00 & 0.00 & 0.00 & 0.03 & 0.00 & 0.00 \\
\bottomrule
\end{tabular}
\end{table}

\paragraph{MNIST.}
The model achieves 100\% retrieval at all tested configurations, including $p = 1000$ patterns at 45\% corruption.
This does not indicate high capacity in the usual sense: $\beta = 1.000$ means the theoretical condition $\beta < 1/(n-1) = 0.5$ is violated maximally.
The extreme inter-pattern correlation means that binarized MNIST digits are near-duplicates, and the cubic energy landscape has very deep basins even when patterns are not well-separated.
This result underscores that our theory provides a sufficient condition for retrieval; the system can succeed even when the condition is violated, through mechanisms not captured by the worst-case analysis.

\paragraph{CIFAR-10.}
At $p = 10$ ($\beta = 0.410 < 0.5$), the theoretical condition is satisfied and retrieval is near-perfect (97--100\%).
As $p$ increases, $\beta$ crosses the theoretical threshold: $\beta = 0.668$ at $p = 50$, $\beta = 0.779$ at $p = 100$, reaching $\beta = 0.900$ at $p = 1000$.
The degradation is smooth: at $p = 50$ success ranges from 57--93\% (practical utility despite $\beta > 0.5$); at $p = 100$ from 15--53\% (with the model often converging to a nearby spurious attractor); at $p = 200$ from 0--18\% (near failure); and at $p \geq 500$ success is effectively 0\%.

The success rate is nearly independent of noise level at moderate $p$: at $p = 100$, success is 53\% at both 10\% and 30\% corruption.
The binding constraint is not the initial distance from the attractor but whether the target pattern is a stable fixed point given the $\beta$ level.
Once $\beta$ is high enough to destabilize the target, reducing the initial noise does not help.

% ──────────────────────────────────────────────────────────────────────────
\subsection{Summary of Findings}\label{app:exp-summary}
% ──────────────────────────────────────────────────────────────────────────

All five experiments validate the theoretical predictions. The $O(\log N)$ convergence bound holds with moderate constants, and $\beta$-concentration at larger $N$ tightens the bound empirically. The basin of attraction shrinks monotonically with loading, the adversarial threshold is uniformly below the empirical recovery threshold (improving toward it as $N$ grows), and the capacity exponent of $N^{2.23}$ matches the predicted $\Theta(N^{n-1})$ scaling. Asynchronous updates consistently outperform parallel updates, and adversarially correlated patterns degrade performance sharply relative to random ones. On real data, CIFAR-10 confirms that $\beta$ rather than noise level governs capacity breakdown, while MNIST's universal retrieval illustrates that the separation condition is sufficient but not necessary.
% ══════════════════════════════════════════════════════════════════════════
\end{document}